\documentclass[sigconf,nonacm]{aamas} 

\usepackage{balance} 
\usepackage{bm}
\usepackage{dsfont}
\usepackage{algorithm}
\usepackage{algpseudocode}
\usepackage{tikz}
\usepackage{enumerate}
\usepackage{blkarray}
\usepackage{paralist}

\theoremstyle{definition}
\newtheorem{dfn}{Definition}[section]
\newtheorem{exmpl}{Example}[section]
\newtheorem{prop}{Proposition}[section]

\newcommand{\R}{\mathbb{R}}
\newcommand{\N}{\mathbb{N}}

\newcommand{\Argmin}{\mathop{\rm argmin}\limits}
\newcommand{\argmax}{\mathop{\rm argmax}\nolimits}

\usepackage{tikz}
\newcommand{\circled}[1]{\tikz[baseline=(char.base)]{\node[shape=circle,draw,inner sep=1pt] (char) {#1};}}
\newcommand{\fround}{r}
\newcommand{\fsoftround}{sr}
\newcommand{\indicator}[1]{\mathds{1}\left[#1\right]}
\newcommand{\allonevec}{\bm{1}}
\newcommand{\allzerovec}{\bm{0}}
\newcommand{\zeromatrix}{\bm{O}}
\newcommand{\softmax}{\mathrm{softmax}}

\newcommand{\rep}{\mathrm{repeat}}

\newcommand{\ceil}[1]{\lceil#1\rceil}
\newcommand{\softRR}{\mathrm{SoftRR}}
\newcommand{\neuralRR}{\mathrm{NRR}}

\newcommand{\setofRR}{\mathcal{F}_{\mathrm{RR}}}

\setcopyright{ifaamas}
\acmConference[AAMAS '25]{Proc.\@ of the 24th International Conference
on Autonomous Agents and Multiagent Systems (AAMAS 2025)}{May 19 -- 23, 2025}
{Detroit, Michigan, USA}{A.~El~Fallah~Seghrouchni, Y.~Vorobeychik, S.~Das, A.~Nowe (eds.)}
\copyrightyear{2025}
\acmYear{2025}
\acmDOI{}
\acmPrice{}
\acmISBN{}

\acmSubmissionID{274}

\title[AAMAS-2025 Formatting Instructions]{Learning Fair and Preferable Allocations through Neural Network}

\author{Ryota Maruo}
\affiliation{
  \institution{Kyoto University}
  \city{Kyoto}
  \country{Japan}}
\email{mryota@ml.ist.i.kyoto-u.ac.jp}

\author{Koh Takeuchi}
\affiliation{
  \institution{Kyoto University}
  \city{Kyoto}
  \country{Japan}}
\email{takeuchi@i.kyoto-u.ac.jp}

\author{Hisashi Kashima}
\affiliation{
  \institution{Kyoto University}
  \city{Kyoto}
  \country{Japan}}
\email{kashima@i.kyoto-u.ac.jp}

\begin{abstract}
The fair allocation of indivisible resources is a fundamental problem.
Existing research has developed various allocation mechanisms or algorithms to satisfy different fairness notions.
For example, round robin (RR) was proposed to meet the fairness criterion known as envy-freeness up to one good (EF1).
Expert algorithms without mathematical formulations are used in real-world resource allocation problems to find preferable outcomes for users.
Therefore, we aim to design mechanisms that strictly satisfy good properties with replicating expert knowledge.
However, this problem is challenging because such heuristic rules are often difficult to formalize mathematically, complicating their integration into theoretical frameworks. 
Additionally, formal algorithms struggle to find preferable outcomes, and directly replicating these implicit rules can result in unfair allocations because human decision-making can introduce biases.
In this paper, we aim to learn implicit allocation mechanisms from examples while strictly satisfying fairness constraints, specifically focusing on learning EF1 allocation mechanisms through supervised learning on examples of reported valuations and corresponding allocation outcomes produced by implicit rules.
To address this, we developed a neural RR (NRR), a novel neural network that parameterizes RR.
NRR is built from a differentiable relaxation of RR and can be trained to learn the agent ordering used for RR.
We conducted experiments to learn EF1 allocation mechanisms from examples, demonstrating that our method outperforms baselines in terms of the proximity of predicted allocations and other metrics.

\end{abstract}


\newcommand{\BibTeX}{\rm B\kern-.05em{\sc i\kern-.025em b}\kern-.08em\TeX}

\begin{document}


\pagestyle{fancy}
\fancyhead{}


\maketitle 


\section{Introduction}
\label{sec:introduction}
The fair allocation of indivisible resources is a fundamental problem in mechanism design, extensively studied in both computer science and economics~\cite{amanatidis2023fair,aziz2019fair,brandt2016handbook}.
Fair division problems aim to allocate indivisible items fairly among agents who have individual preferences or valuations for the resources.
Research has predominantly focused on the fair division of {\em goods}, where $n$ agents assign non-negative values on $m$ indivisible items~\cite{conitzer2017fair}.
Examples of good allocation include course assignments for students~\cite{othman2010finding} and the general goods allocation approach employed by Spliddit~\cite{goldman2015spliddit}, one of the most successful applications of fair division principles.

Researchers have developed numerous allocation mechanisms or algorithms to address fair division problems under various fairness concepts.
One well-known fairness criterion is envy-freeness (EF)~\cite{foley1966resource}, where no agent believes that another agent received a better allocation~\cite{amanatidis2023fair}.
However, EF allocations do not always exist, as demonstrated by a simple counterexample: when there are two agents and a single good, any allocation results in envy from the unallocated agent.
To address this, \citet{lipton2004approximately} and \citet{budish2011combinatorial} proposed a relaxed fairness notion called envy-freeness up to one good (EF1), meaning that envy can be eliminated by removing a single good from the envied agent’s allocation.
EF1 allocations always exist and can be computed in polynomial time using the round robin (RR) mechanism~\cite{caragiannis2016unreasonable}.
In this algorithm, an order of agents is defined, and each agent, in turn, selects their most preferred item from the remaining items.
Other algorithms and variants of fairness notions have also been explored, as described in surveys~\cite{aziz2022algorithmic,amanatidis2023fair}.


Expert algorithms without mathematical formulations are used in real-world resource allocation problems to find preferable outcomes for users.
Although their goodness is not formally proven, these algorithms can use implicit or empirical knowledge in various domains~\cite{gudes1990resource}.
For example, health providers allocate clinical resources based on subconscious knowledge, such as work ethics~\cite{lemieux-charles1993ethical,shaikh2020artificial}.
As in the case of divide-and-choose, which appeared in Bible and was later shown to have good properties in EF, there is a possibility that these algorithms also have some good properties. 
Even if we can acquire expert knowledge, it is not immediately amenable to formalization.
On the other hand, existing formal algorithms cannot skim off the top of implicit knowledge of experts.


Designing mechanisms that strictly satisfy good properties with replicating expert knowledge is challenging for three reasons.
First, because these heuristics are difficult to formalize into precise mathematical expressions, translating them into theoretical frameworks is problematic.
Second, formal algorithms can select an allocation that strictly satisfies good properties, but struggle to find one from a set of candidates that is preferable to experts.
Third, human judgments can introduce biases, potentially leading to unfair allocations~\cite{gordon2017resource,li2010how}, and thus directly mimicking such implicit rules can reproduce or increase undesired biases.


In this paper, we study how to learn implicit allocation rules from examples of allocation results while strictly satisfying fairness constraints.
We significantly extended the general idea introduced by \citet{narasimhan2016automated} within the context of {\em automated mechanism design} (AMD). 
In their work, they proposed a framework that learns mechanisms from examples while ensuring constraint satisfaction.
Given example pairs of reported valuations and allocations based on implicit rules, our goal is to train a parameterized fair allocation mechanism by capturing the relationship between inputs (valuations) and outputs (allocations).
We optimize the parameters through supervised learning, minimizing the discrepancy between predicted and actual outcomes.
This approach extends the prior work~\cite{narasimhan2016automated}, adapting it to our problem, which addresses both the reproduction of implicit rules and the enforcement of fairness constraints.
Instead of formalizing expert knowledge, our approach extracts their rules as a parameter through learning from examples.

We focused on learning EF1 allocation mechanisms via supervised learning using examples of reported valuations and corresponding allocation outcomes determined by implicit rules.
We selected EF1 as the fairness constraint because it is widely used in real-world applications like Spliddit~\cite{goldman2015spliddit}.
To implement the framework, we aimed to construct a parameterized family of EF1 allocation mechanisms by developing a neural network that parameterizes the RR mechanism.
In particular, we treated the agent order as a parameter of the RR mechanism and aimed to develop a neural network in which this order is a learnable parameter.
To implement this approach, we introduced two novel techniques.
First, unlike prior work~\cite{narasimhan2016automated}, we proposed a soft RR ($\softRR$) algorithm that makes the discrete procedure of RR differentiable, enabling it to be used for back-propagation.
Second, we constructed a novel neural network called a neural RR ($\neuralRR$).
Given a valuation profile as input, $\neuralRR$ first computes the agent order parametrically, and then executes $\softRR$ to output an EF1 allocation.
This architecture allows the agent order of RR to be learned by back-propagating errors through $\softRR$, rather than being pre-specified or fixed.
$\neuralRR$ rigorously satisfies EF1 during inference because it is equivalent to RR with the parametrically computed order.

We conducted experiments to learn EF1 allocation mechanisms from examples.
We synthesized allocation examples by sampling valuation profiles and executing an existing allocation mechanism.
To evaluate performance, we measured discrepancies between the predicted and correct allocation outcomes.
Additionally, we used other relevant metrics to verify that $\neuralRR$ accurately reflects the implicit objectives encoded in the examples.
We used RR and an existing neural network model as baselines.
Experimental results demonstrated that $\neuralRR$ outperforms the baselines in terms of the proximity of predicted allocations and other metrics.

Our contributions are summarized as follows.
(i) To the best of our knowledge, we are the first to consider learning EF1 allocation rules from examples within the fair division and AMD literature.
(ii) We proposed $\softRR$, an algorithm that makes RR differentiable and enables backpropagation of predicted and actual allocation outcomes.
(iii) We developed a novel neural network, called $\neuralRR$, for learning EF1 allocation mechanisms through examples.
$\neuralRR$ is constructed from $\softRR$ and can learn the agent order used for RR.
(iv) We conducted experiments to learn an EF1 allocation mechanism from reported agent valuations and corresponding allocation outcomes.
The experimental results confirmed that $\neuralRR$ can recover implicit allocation mechanisms from examples while satisfying fairness constraints.

Code is available at \url{https://github.com/MandR1215/neural_rr}.


\section{Related Work}
\label{sec:related_works}
The study of fair allocation of indivisible items has led to the development of various algorithms based on different fairness and efficiency notions.
In addition to EF, other fairness concepts include the maximin share (MMS), which ensures that each agent receives an allocation at least as valuable as the least-valued subset of items they could obtain, assuming they divide the entire set and select the least valued subset for themselves~\cite{budish2011combinatorial}.
\citet{amanatidis2017approximation} proposed an approximation algorithm for MMS allocations.
Pareto efficiency is another key efficiency notion, meaning that 
no alternative allocation can make some agents strictly better off without making any other agent strictly worse off.
\citet{caragiannis2016unreasonable} demonstrated that allocations maximizing Nash welfare yield both EF1 and Pareto-efficient allocations.
However, these formal algorithms are static and cannot learn from examples.

Our work contributes to the literature on AMD, a field first introduced by \citet{conitzer2002complexity, conitzer2004self}.
AMD focuses on automatically designing mechanisms by solving optimization problems, where objective functions correspond to social objectives and constraints model incentive properties~\cite{conitzer2002complexity,conitzer2004self,sandholm2003automated}.
For instance, in auction settings, the problem is often framed as maximizing expected revenue while satisfying conditions such as incentive compatibility and individual rationality~\cite{conitzer2004self,curry2023differentiable,duan2023scalable,duan2022context,duetting2019optimal,feng2018deep,peri2021preferencenet,rahme2021permutation,sandholm2015automated,shen2019automated,wang2024gemnet}.
AMD also has many other applications, including facility location~\cite{golowich2018deep}, data market design~\cite{ravindranath2023datamarket}, and contract design~\cite{wang2023contract}.
However, these methods are typically fixed to explicit mathematical optimization problems and do not consider fitting implicit rules from examples.

In differentiable economics \cite{duetting2019optimal}, researchers have designed neural networks to solve AMD problems.
Pioneered by \citet{duetting2019optimal}, who introduced RegretNet, a line of AMD research focused on solving revenue-optimal auction problems using neural networks~\cite{curry2023differentiable,feng2018deep,peri2021preferencenet,rahme2021permutation,shen2019automated,wang2024gemnet}.
However, the current study focuses on fair allocations.
\citet{mishra2022eef1nn} explored AMD for fair allocations using a neural network, but their method only approximately satisfies EF1, whereas ours rigorously satisfies the condition.


\section{Preliminaries}
\label{sec:Preliminaries}
We use $[n]$ to denote the set $\{1, \dots, n\}$ for $n \in \mathbb{N}$.
A row vector is represented as $\bm{x} = [x_1, \dots, x_d]$, and a matrix as $\bm{X} = [X_{ij}]_{ij}$. 
The $d$-dimensional all-one vector and all-zero vector are denoted by $\allonevec_d$ and $\allzerovec_d$, respectively. 
The element-wise product of two vectors, $\bm{x}$ and $\bm{y}$, is written as $\bm{x} \odot \bm{y} = [x_1 y_1, \dots, x_d y_d]$. 
We define $\bm{x} \ge \bm{y} \iff x_i \ge y_i , (\forall i)$. 
A zero matrix with $n$ rows and $m$ columns is denoted by $\zeromatrix_{n, m}$. 
For a matrix $\bm{X} \in \R^{n \times m}$, $\bm{X}[i]$ denotes the $i$-th row vector, $\bm{X}[:, j]$ denotes the $j$-th column vector, and $\bm{X}[a:b, c:d] = [X_{ij}]_{i=a, \dots, b, j = c, \dots, d}$ denotes the sub-matrix selected by row and column ranges. 
We use $\indicator{\cdot}$ to denote the indicator function.

We study the standard setting of fair division of a set of indivisible goods $[m] = \{1,2,...,m\}$ among a set of agents $[n] = \{1,2,\dots, n\}$.
A {\em bundle} is a subset of goods.
An agent $i$ has a {\em valuation function} $v_i:2^{[m]}\to\R_{\ge 0}$ that assigns a non-negative real value to a bundle.
We assume that the valuation function is additive; that is, we define $v_i(S) := \sum_{j\in S} v_{ij}$ for each bundle $S\subseteq[m]$ where $v_{ij} := v_i(\{j\})$.
A {\em valuation profile} $(v_1,\dots,v_n)$ is a collection of valuation functions, and we represent it by a matrix $\bm{V} := [v_{ij}]_{i\in [n], j\in [m]}\in\R_{\ge 0}^{n\times m}$.
We denote an {\em integral allocation} by $\mathcal{A} = (A_1,\dots, A_n)$, where $A_i\subseteq[m]$ is the bundle allocated to agent $i$, and each good is allocated to exactly one agent, i.e., $A_i\cap A_j = \emptyset$ for all $i\neq j$ and $\cup_{i\in [n]} A_i = [m]$.
We also represent an integral allocation $(A_1,\dots, A_n)$ by the matrix $\bm{A}\in\{0,1\}^{n\times m}$, where $A_{ij} = \indicator{j\in A_i}$.
A {\em fractional allocation} is one in which some goods are allocated fractionally among agents.
Unless otherwise specified, the term ``allocation'' refers to an integral allocation without explicitly using the word ``integral.''

We focus on {\em envy-freeness} (EF) and its relaxation as fairness concepts.
An allocation is EF if every agent values their own allocation at least as highly as they value any other agent's allocation.

\begin{dfn}[EF~\cite{foley1966resource}]
An allocation $(A_1,\dots, A_n)$ is {\em envy-free} (EF) if, for all agents $i,j\in [n]$, $v_i(A_i)\ge v_i(A_j)$.
\end{dfn}

EF allocations do not always exist. 
For example, when allocating one good to two agents, the agent without an allocated good envies the agent who has it.
To guarantee the existence of a solution, \citet{lipton2004approximately} and \citet{budish2011combinatorial} proposed the relaxation of EF known as EF1.
An allocation is EF1 if either no envy exists, or the envy from $i$ to $j$ can be eliminated by removing a good from $A_j$.

\begin{dfn}[EF1~\cite{lipton2004approximately,budish2011combinatorial}]
    An allocation $(A_1,\dots,A_n)$ is EF1 if, for all $i,j\in [n]$, either $v_i(A_i)\ge v_i(A_j)$, or there exists a good $o \in A_j$ such that $v_i(A_i) \ge v_i(A_j\setminus\{o\})$.
\end{dfn}

An {\em allocation mechanism} is a function $f$ that maps a valuation profile to an allocation.
That is, for any profile $\bm{V}$, $f(\bm{V}) = (A(\bm{V})_1,\dots, A(\bm{V})_n)$ is an allocation.
We say that $f$ is EF1 if $f$ always outputs EF1 allocations for any input valuation profile $\bm{V}$.


\section{Problem Setting}
Given an implicit allocation mechanism $g$, our goal is to find an allocation mechanism that approximates $g$, subject to EF1 constraint.

We formally define our problem by following the framework established in prior work~\cite{narasimhan2016automated}.
Because $g$ is an implicit rule, akin to a human heuristic, its explicit formulation of $g$ is unavailable.
Instead, we assume access to $g$ through a dataset $S:=\{(\bm{V}^1, \bm{A}^1),$ $\dots, (\bm{V}^L, \bm{A}^L)\}$, where $\bm{V}^1,\dots,\bm{V}^L$ are examples of valuation profiles sampled from an unknown distribution over the set of all valuation profiles, and $\bm{A}^1 = g(\bm{V}^1),\dots,\bm{A}^L= g(\bm{V}^L)$ are the corresponding allocation outcomes determined by $g$.
Given the dataset $S$, our goal is to find the EF1 allocation mechanism that best approximates $g$:
\begin{align*}
    \min_{f\in\mathcal{F}_\mathrm{EF1}} \sum^L_{r=1} d(\bm{A}^r, f(\bm{V}^r)),
\end{align*}
where $\mathcal{F}_\mathrm{EF1}$ is the set of all EF1 allocation mechanisms, and $d(\bm{A}, \bm{A}')$ is a function that calculates the discrepancy between two allocation outcomes $\bm{A}$ and $\bm{A}'$.
Because the set of all the EF1 allocation mechanisms is not explicitly identifiable, we follow existing research~\cite{narasimhan2016automated} and focus on searching over a subset of EF1 allocation mechanisms.
Technically, we consider a parameterized subset of all the EF1 allocation mechanisms $\mathcal{F} := \{f_{\bm{\theta}}\mid \bm{\theta}\in\bm{\Theta}\}\subset\mathcal{F}_{\mathrm{EF1}}$, where $\bm{\theta}$ is a parameter from the parameter space $\bm{\Theta}$.
The problem is then solved by searching for $f_{\bm{\theta}^*}\in\mathcal{F}$ corresponding to the optimal $\bm{\theta}^*$:
\begin{align}
    f_{\bm{\theta}^*} := \Argmin_{\bm{\theta}\in\bm{\Theta}} \sum^L_{r=1} d(\bm{A}^r, f_{\bm{\theta}}(\bm{V}^r))\label{eq:sample_approx}.
\end{align}
In other words, we optimize the parameter $\bm{\theta}$ by minimizing the discrepancy between the predicted and the actual allocations. 


\section{Proposed Method}
To solve problem in Equation~\eqref{eq:sample_approx}, we propose a parameterized family of mechanisms $\mathcal{F}$ based on RR~\cite{caragiannis2016unreasonable}, one of the EF1 allocation algorithms.
RR allocates goods in multiple {\em rounds}, where, in each round, agents choose their most preferred goods from the remaining available items, following a specific order.
RR's output depends on the order of agents, so we propose modeling $f_{\bm{\theta}}$ in Equation~\eqref{eq:sample_approx} through a neural network, where $\bm{\theta}$ as a learnable parameter.
Specifically, we aim to construct a neural network that computes the agent order parametrically and then executes RR according to the order. 
By doing so, we can optimize the agent order by backpropagating errors through RR.
Creating such a model is not straightforward because RR is a discrete procedure and is not directly suitable for gradient-based training.
To address this challenge, we first propose a differentiable relaxation of RR, which incorporates the agent order.
We then develop a neural network that integrates this relaxed RR, along with a sub-network to parametrize the agent order.

We describe our proposed method as follows.
First, we briefly review the RR algorithm.
Next, we present the differentiable relaxation of RR.
Finally, we describe our neural network architecture, which incorporates the relaxed version of RR and a component for parametric computation of agent orders as the two building blocks.

\subsection{Round Robin}
\begin{algorithm}[t]
    \caption{RR~\cite{caragiannis2016unreasonable}}
    \label{alg:rr}
    \begin{algorithmic}[1]
        \Require A valuation profile $\bm{V}=(v_{ij})_{i\in [n], j\in [m]}\in\R_{\ge 0}^{n\times m}$.
        \Ensure An allocation $\mathcal{A}$.
        \For{each agent $i\in [n]$}
            \State $A_i\leftarrow\emptyset$
        \EndFor
        \State $C\leftarrow[m]$
        \For{$r = 1,\dots, \ceil{m/n}$}
            \For{$i=1,\dots,n$}
                \Comment{Run the $r$-th round}
                \If{$C\neq \emptyset$}
                    \State $g^*\leftarrow\argmax_{j\in C} v_{ij}$
                    \State $A_i\leftarrow A_i\cup\{g^*\}$
                    \State $C\leftarrow C\setminus\{g^*\}$
                \EndIf
            \EndFor
        \EndFor
        \State\Return $\mathcal{A} = (A_1,\dots, A_n)$
    \end{algorithmic}
\end{algorithm}

RR consists of multiple rounds, during each of which agents $1,\dots,n$ sequentially pick their most preferred goods according to a predefined order.
The entire procedure of RR is detailed in Algorithm~\ref{alg:rr}.
It can be shown through a straightforward proof that RR always produces an EF1 allocation.

\begin{prop}[\citet{caragiannis2016unreasonable}]
    RR computes an EF1 allocation.
    \label{prop:RR_is_EF1}
\end{prop}
\begin{proof}
    Consider two distinct agents $i$ and $j$.
    Without loss of generality, assume $i < j$.
    Because $i$ picks goods before agent $j$ in each round, agent $i$ receives more valuable goods than those allocated to agent $j$.
    As a result, agent $i$ has no envy toward agent $j$.
    Envy may exist from $j$ toward $i$.
    Now, consider the moment when $i$ selects the first good $o$ in the initial round. 
    If we treat the execution of RR process for the remaining goods as a new process, then $j$ picks before $i$ in each subsequent round, eliminating $j$'s envy toward $i$. 
    As a result, $j$’s envy dissipates once $o$ is removed from $A_i$.
\end{proof}

RR can produce different EF1 allocations depending on the initial order of the agents. 
In other words, permuting the indices of the agents can result in different allocations as shown in Example~\ref{exmpl:drr_order}.

\begin{exmpl} 
\label{exmpl:drr_order}
\begin{table}[tb]
    \centering
    \caption{The valuation profile considered in Example~\ref{exmpl:drr_order}.
    The first row is the index of goods.
    The rest of three rows correspond to the valuation of the agents.
    The circled numbers mean the good in the column is allocated to the agent in the same row.
    Left: the original valuation profile.
    Right: the permuted valuation profile. The agents $3$, $1$, $2$ are treated as the agents $1'$, $2'$, $3'$, respectively.}
    \label{tab:example_v_drr_order}
    \begin{tabular}{c|cccc}
        & $1$ & $2$ & $3$ & $4$\\\midrule
        $v_1$& $1$ & $0$ & $\circled{3}$ & $\circled{2}$\\
        $v_2$& $\circled{3}$ & $2$ & $1$ & $0$ \\
        $v_3$& $4$ & $\circled{3}$ & $2$ & $1$ \\\bottomrule
    \end{tabular}\hfill
    \begin{tabular}{c|cccc}
        & $1$ & $2$ & $3$ & $4$\\\midrule
        $v_{1'}=v_3$& $\circled{4}$ & $3$ & $2$ & $\circled{1}$\\
        $v_{2'}=v_1$& $1$ & $0$ & $\circled{3}$ & $2$ \\
        $v_{3'}=v_2$& $3$ & $\circled{2}$ & $1$ & $0$ \\\bottomrule
    \end{tabular}
\end{table}
Consider three agents $1$, $2$, and $3$ whose valuations $v_1$, $v_2$, and $v_3$, respectively, are as described in the right Table~\ref{tab:example_v_drr_order}.
When we run RR in that agent order, we obtain an EF1 allocation $\mathcal{A} = (A_1 = \{3,4\}, A_2 = \{1\}, A_3 = \{2\})$.
On the other hand, consider permuting the agent indices: set three virtual three agents $1'$, $2'$ and $3'$ who are actually agent $3$, $1$, and $2$, respectively.
When we run RR on the three virtual agents $1'$ , $2'$, and $3'$ in this order, we obtain an allocation $(A_{1'} = \{1,4\}, A_{2'} = \{3\}, A_{3'}=\{2\})$, which is actually equal to allocation $(A_{1}' = \{3\}, A_{2}' = \{2\}, A_{3}'=\{1,4\}) \neq \mathcal{A}$.\qed
\end{exmpl}

We formally define RR that runs on different agent indices.
A permutation is represented by a bijective function $\pi:[n]\to[n]$.

\begin{dfn}[RR induced by permutations]
    Let $\pi$ be a permutation. 
    The procedure $\mathrm{RR}_\pi:\bm{V}\mapsto\mathcal{A}$ represents the RR procedure that operates on a valuation profile $(v_{\pi^{-1}(1)}, \dots, v_{\pi^{-1}(n)})$ according to this order.
    In other words, $\mathrm{RR}_\pi$ is the RR procedure applied to virtual agents $1',\dots,n'$, where the valuations are $v_{1'} = v_{\pi^{-1}(1)},\dots, v_{n'} = v_{\pi^{-1}(n)}$.
    We refer to $\mathrm{RR}_\pi$ as ``RR induced by $\pi$''.
\end{dfn}

Proposition~\ref{prop:RR_is_EF1} holds regardless of the agent index order by simply modifying the assumption $i < j$ to $i' < j'$.
Therefore, $\mathrm{RR}_\pi$ forms a subset of EF1 allocations.

\begin{prop}
    Let $\setofRR := \{\mathrm{RR}_\pi\mid \text{$\pi$ is a permutation}\}$ be the set of all RR mechanisms induced by some permutation.
    Then, $\setofRR \subset \mathcal{F}_{\mathrm{EF1}}$.
\end{prop}

\subsection{Differentiable Relaxation of RR}
To search over a subset of EF1 allocations as described in Equation~\eqref{eq:sample_approx}, we aim to represent $f_{\bm{\theta}}$ using a neural network based on RR. 
Specifically, we seek to construct a neural network that parameterically computes the agent order and executes RR within the network, allowing us to tune the parameters to optimize the agent order.
To achieve this, we must consider the differentiable relaxation of RR, enabling error backpropagation from predicted allocations.
Without such relaxation, the discrete nature of RR makes it unsuitable as a direct network layer.

First, we develop the differentiable relaxation of a single round of RR.
Then, we introduce $\softRR$, an algorithm for differentiable relaxation of RR.

\subsubsection{Differentiable Relaxation of One Round}
\begin{algorithm}[t]
    \caption{Differentiable Relaxation of One Round}
    \label{alg:continuous_round}
    \begin{algorithmic}[1]
        \Require A valuation profile $\bm{V}\in\R_{\ge 0}^{n\times m}$.
        \Ensure A matrix $\bm{R}\in\R^{n\times m}$.
        \Function{$\fsoftround_\tau$}{$\bm{V}$}
            \State $\bm{R} \leftarrow \zeromatrix_{n,m}$
            \State $\bm{c} \leftarrow \allonevec_m$
            \For{$i = 1,\dots,n$}
                \State $\bm{y} \leftarrow \softmax((\bm{V}[i] - \min(\bm{V}[i])\cdot\allonevec + \allonevec) \odot \bm{c} / \tau)$\label{algline:continuous_round_sety}
                \State $\bm{c} \leftarrow (\allonevec - \bm{y})\odot\bm{c}$\label{algline:continuous_round_updatec}
                \State $\bm{R}[i] \leftarrow \bm{y}$
            \EndFor
            \State \Return $\bm{R}$
        \EndFunction
    \end{algorithmic}
\end{algorithm}
We first formally define one round.
Consider a scenario with $n$ agents and $m$ goods.
Let $\bm{V}=(v_{ij})_{i\in [n], j\in [m]}$ represent the valuation of the $n$ agents over the $m$ available goods.
The allocation obtained after one round is denoted as $\fround:\bm{V}\mapsto\bm{A}$.
The resulting allocation from $\fround$ is:
\begin{align}
    \fround(\bm{V}) = \left[\indicator{j = \argmax_{j'\in C_i}\{v_{ij'}\}}\right]_{i\in[n],j\in[m]}\label{eq:fround},
\end{align}
where the $\argmax$ operator breaks ties in favor of the earlier index. 
The set $C_i$ is defined as 
\begin{align*}
    C_i := \left\{
        \begin{array}{ll}
            [m] & (\text{if $i = 1$}) \\
            C_{i-1}\setminus\left\{\argmax_{j'\in C_{i-1}}\{v_{ij'}\}\right\} & (\text{otherwise})
        \end{array}
    \right..
\end{align*}
That is, $C_i$ represents the set of available goods left for the agent $i$ after agents $1,2,\dots,i-1$ have selected their most preferred goods.
Using $C_i$, Equation~\eqref{eq:fround} calculates the resulting allocation obtained from one round.

Next, we present the differentiable relaxation of the function $\fround(\cdot)$.
The computation is detailed in Algorithm~\ref{alg:continuous_round}.
Our relaxed function, denoted as  $\fsoftround_\tau$, includes a {\em temperature} parameter $\tau > 0$.
Intuitively, we replace the $\argmax$ operator with a $\softmax$ that incorporates the temperature parameter $\tau$.
The set $C_i$ is represented by a vector $\bm{c}$ where each element $c_j\in [0,1]$ satisfies $c_j\approx 1$ if $j\in C_i$, and $c_j\approx 0$ otherwise.
To simulate the $\argmax$ operator over the remaining goods, we apply $\softmax$ to the expression $(\bm{V}[i] - \min(\bm{V}[i])\cdot\allonevec + \allonevec)\odot\bm{c}$.
The term $(- \min(\bm{V}[i])\cdot\allonevec + \allonevec)$ distinguish between remaining goods and those already taken: the $j$-th element becomes approximately greater than $1$ if $c_j\approx 1$, while it remains close to zero if $c_j\approx 0$.

The parameter $\tau$ controls the approximation precision. 
We formally prove that $\fsoftround_\tau$ converges to $\fround$ in the limit as $\tau\to+0$.

\begin{prop}
    Let $\bm{V}\in\R^{n\times m}$ be a valuation profile.
    Assume $n\le m$ and assume there are no ties in any row of $\bm{V}$.
    That is, $\forall i\in [n], \forall j,j'\in [m], j\neq j'\implies V_{ij} \neq V_{ij'}$.
    Then, 
    \begin{align*}
        \lim_{\tau\to +0} \fsoftround_\tau(\bm{V}) = \fround(\bm{V}).
    \end{align*}
    \label{prop:soft_one_round}
\end{prop}
\begin{proof}
    We set the following two loop invariants for the for-loop:
    \begin{enumerate}[(L1)]
        \item After the $i$-th iteration, $\bm{R}[i] = \fround(\bm{V})[i]$.
        \item After the $i$-th iteration, for all $j\in[m]$, $c_j = \indicator{j\in C_{i+1}}$.
    \end{enumerate}
    
    These invariants hold trivially before the for-loop.
    
    Now, consider entering the $i$-th iteration, assuming (L1) and (L2) hold for the $(i-1)$-th iteration.
    Let $\bm{z} = (\bm{V}[i] - \min(\bm{V}[i])\cdot\allonevec+\allonevec)\odot\bm{c}$.
    Because $\bm{V}[i] - \min(\bm{V}[i])\cdot\allonevec+\allonevec\ge\allonevec$ and (L2) holds, 
    \begin{align*}
        z_j = \left\{
            \begin{array}{ll}
               v_{ij} - \min(\bm{V}[i]) + 1 > 0 & (\text{if $j\in C_i$}) \\
               0 & (\text{otherwise})
            \end{array}
        \right..
    \end{align*}
    In addition, for any $\bm{x}\in\R^d$ with no ties, we have
    \begin{align}
        \lim_{\tau\to +0}\softmax(\bm{x}/\tau) = \left[\indicator{j = \argmax_{j'\in [d]} x_{j'}}\right]_j.\label{eq:softmax_approx}
    \end{align}
    Therefore, the vector $\bm{y}$ at Line~\ref{algline:continuous_round_sety} is
    \begin{align*}
        \bm{y} 
        &= \lim_{\tau\to +0} \softmax(\bm{z}/\tau)\\
        &= \left[\indicator{j = \argmax_{j'\in [m]}\{z_{j'}\}}\right]_j\\
        &= \left[\indicator{j = \argmax_{j'\in C_i}\{v_{ij'} - \min(\bm{V}[i]) + 1\}}\right]_j\\
        &= \left[\indicator{j = \argmax_{j'\in C_i}\{v_{ij'}\}}\right]_j\\
        &= \fround(\bm{V})[i].
    \end{align*}
    Thus, the loop invariant (L1) holds.
    Furthermore,
    \begin{align*}
        (\allonevec - \bm{y})\odot\bm{c}
        &= \left[\indicator{j \neq \argmax_{j'\in C_i}\{v_{ij'}\}}\right]_j\odot \left[\indicator{j \in C_i}\right]_j\\
        &= \left[\indicator{j \in C_i\setminus\left\{\argmax_{j'\in C_i}\{v_{ij'}\}\right\}}\right]_j.
    \end{align*}
    This means the update $\bm{c} \leftarrow (\allonevec - \bm{y})\odot\bm{c}$ at Line~\ref{algline:continuous_round_updatec} leads to the vector representation of $C_{i+1}$.
    Therefore, loop invariant (L2) holds for the $i$-th iteration.

    After the $n$-th iteration, we obtain $\bm{R} = r(\bm{V})$ by maintaining the loop invariant (L1).
\end{proof}

The two assumptions in Proposition~\ref{prop:soft_one_round} are necessary.
The `no-tie' condition for rows of $\bm{V}$ is required because the 
Equation~\eqref{eq:softmax_approx} only holds when the elements of $\bm{x}$ have no ties.
Similarly, the assumption $n \le m$ is essential because if $n > m$, then $\bm{c} \approx \bm{0}$ after the $m$-th iteration. 
As a result, the vector $\left(\bm{V}[i] - \min(\bm{V}[i]) \cdot \bm{1} + \bm{1} \right) \odot \bm{c} / \tau \approx \bm{0}$, leading to multiple zero ties, which prevents convergence to the hard $\argmax$.

\subsubsection{Differentiable Relaxation of RR}
\begin{algorithm}[t]
    \caption{$\softRR_\tau$}
    \label{alg:soft_rr}
    \begin{algorithmic}[1]
        \Require A valuation profile $\bm{V}\in\R_{\ge 0}^{n\times m}$
        \Ensure A matrix $\bm{R}\in\R^{n\times m}$
        \State $k\leftarrow \ceil{m/n}$ \label{algline:soft_rr_setk}
        \State $\bm{V}_{\mathrm{rep}} \leftarrow \rep(\bm{V}, k)$\label{algline:soft_rr_repeatV}
        \State $\bm{R}_{\mathrm{rep}}\leftarrow \fsoftround_\tau(\bm{V}_{\mathrm{rep}})$\label{algline:soft_rr_rounds}
        \State Split $\bm{R}_{\mathrm{rep}}$ into $k$ matrices: $\bm{R}_1\leftarrow\bm{R}_{\mathrm{rep}}[1:n, 1:m], \bm{R}_2\leftarrow\bm{R}_{\mathrm{rep}}[n+1:2n, 1:m],\dots, \bm{R}_k\leftarrow\bm{R}_{\mathrm{rep}}[(k-1)n+1:kn, 1:m]$\label{algline:soft_rr_splitR}
        \State $\bm{R}\leftarrow \sum_{r=1}^k \bm{R}_r$\label{algline:soft_rr_sumR}
        \State \Return $\bm{R}$
    \end{algorithmic}
\end{algorithm}

\begin{figure}[t]
    \centering
    \includegraphics[width=\linewidth]{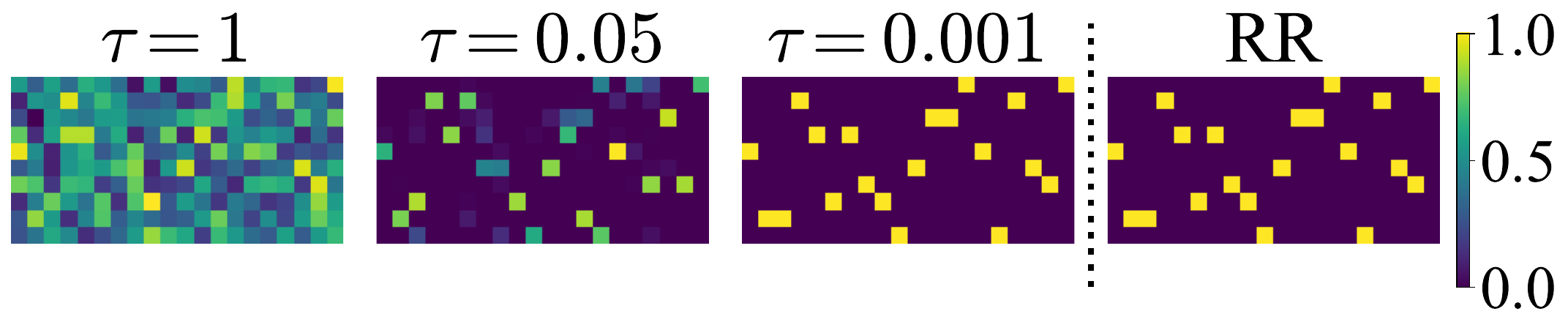}
    \caption{The convergence of $\softRR_\tau$. 
    The first three figures from the left show the outputs of $\softRR_\tau$ for $\tau=1$, $0.05$, and $0.001$, respectively.
    The rightmost figure is the output of RR.}
    \label{fig:softrr_tau}
    \Description{Convergence of SoftRR to RR.}
\end{figure}

Using $\fsoftround_\tau(\cdot)$ as a building block, we propose $\softRR_\tau$, the algorithm that makes RR differentiable and enables it for backpropagation.
The pseudo-code for this process is described in Algorithm~\ref{alg:soft_rr}.

Intuitively, $\softRR_\tau$ approximates the original RR by converting the multiple rounds into a single round with copied agents.
To achieve this, we repeat $\bm{V}$ by $k=\ceil{m/n}$ times along the row direction using the function, defined as
\begin{align*}
    \bm{V}_{\mathrm{rep}} = \rep(\bm{V},k) := \left.\begin{bmatrix} \bm{V}\\\vdots\\\bm{V}\end{bmatrix}\right\} (\text{$k$ times}).
\end{align*}
We apply $\fsoftround_\tau$ to $\bm{V}_\mathrm{rep}$.
The operations in Lines~\ref{algline:soft_rr_setk} to \ref{algline:soft_rr_rounds} simulate  each agent $i$ being replicated into $k$ distinct agents, with each of $k$ agents receiving goods individually.
Finally, we sum the $k$ allocations to consolidate them into a single agent in Lines~\ref{algline:soft_rr_splitR} and \ref{algline:soft_rr_sumR}.

$\softRR_\tau$ incorporates the temperature parameter $\tau$ from $\fsoftround_\tau$.
We achieve results analogous to those in Proposition~\ref{prop:soft_one_round} as $\tau\to +0$.

\begin{prop}
    Let $\bm{V}\in\R_{\ge 0}^{n\times m}$ be a valuation profile.
    Assume $m\mod n = 0$ and no ties exist in any row of $\bm{V}$.
    Then, 
    \begin{align*}
        \lim_{\tau\to +0} \softRR_\tau(\bm{V}) = \mathrm{RR}(\bm{V}).
    \end{align*}
\end{prop}
\begin{proof}
    Let $k = m/n\in\N$.
    Because $\bm{V}_{\mathrm{rep}}\in\R_{\ge 0}^{m\times m}$ has no ties in any of its rows, $\bm{R}_{\mathrm{rep}} = \fsoftround_\tau(\bm{V}_{\mathrm{rep}})$ computes the allocation result of one round for $kn$ distinct agents and $m$ goods exactly in the limit of $\tau\to+0$, by Proposition~\ref{prop:soft_one_round}.
    Hence, Lines~\ref{algline:soft_rr_splitR} and \ref{algline:soft_rr_sumR} produces the final allocation, which is identical to that of RR.
\end{proof}
We present an example of the convergence of $\softRR_\tau$ with respect to $\tau$ in Figure~\ref{fig:softrr_tau}.
We independently sampled $v_{i,j}\sim U[0,1]$, generating a valuation profile $\bm{V}\in\R_{\ge 0}^{10\times 20}$.
As the parameter $\tau$ decreases, 
$\softRR_\tau$ converges to the output of RR.

\subsection{NeuralRR}
\begin{figure*}[t]
    \centering
    \includegraphics[width=0.9\linewidth]{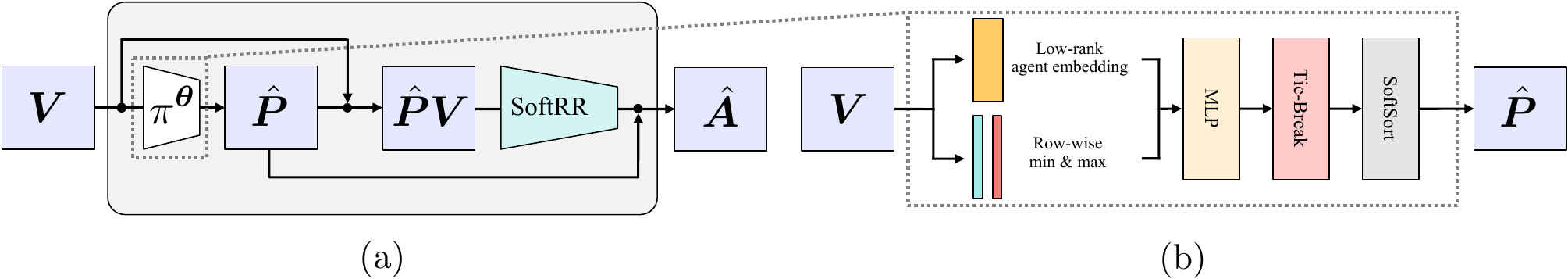}
    \caption{The architecture of $\neuralRR$. 
    (a) Overall architecture: 
    The input valuation $\bm{V}$ as an input, is fed into the network, where a permutation matrix $\hat{\bm{P}}$ is computed. 
    This matrix is then multiplied by $\bm{V}$, and $\softRR$ is executed to obtain the allocation result. (b) Sub-network $\pi^{\bm{\theta}}$: This sub-network computes the permutation matrix $\hat{\bm{P}}$ from the input valuation $\bm{V}$.
    }
    \label{fig:neural_drr}
    \Description{Architecture of the proposed neural network.}
\end{figure*}

Since $\softRR$ is differentiable, we can optimize the agent order via backpropagation.
Using $\softRR$ as a building block, we propose $\neuralRR$, a novel neural network architecture that models $f_{\bm{\theta}}$ in Equation~\eqref{eq:sample_approx}.
Given a valuation profile as input, $\neuralRR$ computes the agent order parameterically, and then executes $\softRR$ to produce a fractional allocation.
The architecture is shown in Figure~\ref{fig:neural_drr}.

Our architecture models RR induced by a permutation parameterically computed from the input valuation.
Specifically, $\neuralRR$ models $\mathrm{RR}_{\pi^{\bm{\theta}}(\bm{V})}(\bm{V})$, where the input valuation profile $\bm{V}$ is transformed into a permutation $\hat{\pi} = \pi^{\bm{\theta}}(\bm{V})$ by a sub-network $\pi^{\bm{\theta}}$ with learnable parameters $\bm{\theta}$. 
The resulting allocation is then obtained through $\softRR$ using $\hat{\pi}$ and $\bm{V}$.
Because $\softRR$ supports backpropagation, the architecture can learn the agent order by minimizing the output errors and search over $\setofRR$ to find the solution in Equation~\eqref{eq:sample_approx}.

The computation of $\neuralRR$ proceeds as follows.
First, the sub-network $\pi^{\bm{\theta}}(\bm{V})$ computes a permutation matrix $\hat{\bm{P}}\in\R^{n\times n}$ representing $\hat{\pi}$ from the valuation profile $\bm{V}$.
To do this, $\pi^{\bm{\theta}}$ applies singular value decomposition to $\bm{V}$ to obtain agent-specific low-rank embeddings.
Then, the row-wise minimum and maximum values of $\bm{V}$ are concatenated to the embeddings.
These min and max values are explicitly calculated, as they represent fundamental features not easily captured by permutation-invariant models with a fixed-dimensional latent spaces, such as DeepSets~\cite{wagstaff2022universal,zaheer2017deepsets}.
Next, each row is fed into a multi-layer perceptron to project it into a $1$-dimensional space, resulting in an $n$-dimensional vector.
To break ties among the values of this vector, we apply the tie-breaking function $\mathrm{TieBreak}(\bm{a}) := \bm{a} + \mathrm{rank}(\bm{a})$, which adds $\bm{a}\in\R^n$ to $\mathrm{rank}(\bm{a}) := [\#\{j\mid \text{$a_j < a_i$ or $a_j = a_i$ and $j < i$} \}]_i$.\footnote{
   This function is not strictly differentiable with respect to $\bm{a}$, but we detach the term $\mathrm{rank}(\bm{a})$ from the computational graph and treat it as a constant.
}
Finally, the permutation matrix $\hat{\bm{P}}$ is computed using SoftSort~\cite{prillo2020softsort}: $\mathrm{SoftSort}_{\tau'}(\bm{a})=\softmax\left(\frac{-(\mathrm{sort}(\bm{a})^\top\allonevec - \allonevec^\top\bm{a})^2}{\tau'}\right)$, where $\bm{X}^2 := [X_{i,j}^2]_{i,j}$ represents the element-wise square, and the $\softmax$ is applied row-wise.
$\allonevec$ has the same dimension as $\bm{a}$.
SoftSort provides a continuous relaxation of $\mathrm{argsort}$ operator, computing a permutation matrix that sorts the input vector $\bm{a}$.
After computing $\hat{\bm{P}}$, we multiply it by the input valuation to reorder the agents.
This gives the allocation result $\softRR(\hat{\bm{P}}\bm{V})$.
Next, we compute $\hat{\bm{P}}^\top\softRR(\hat{\bm{P}}\bm{V})$ to restore the original agent order, and normalize each column to obtain the final matrix $\hat{\bm{A}}\in\R^{n\times m}$.
Because $\hat{\bm{A}}$ is normalized column-wise, all goods are fractionally allocated to the agents.

$\neuralRR$ operates differently during training and inference.
Specifically, during training, $\neuralRR$ produces fractional allocations, whereas during inference, it generates integral allocations because we use hard computations for both $\hat{\bm{P}}$ and $\softRR$.
As a result, during inference, $\neuralRR$ is equivalent to $\mathrm{RR}_{\pi^{\bm{\theta}}(\bm{V})}(\bm{V})$, and thus it rigorously satisfies EF1 for all parameters $\bm{\theta}$.
While NRR relaxes the problem in Equation~\eqref{eq:sample_approx} by outputting fractional allocations during training, it still adheres to the original problem because it can be used as an EF1 allocation mechanism at any point during training.

\subsection{Loss Function}
To train $\neuralRR$, we use a column-wise cross-entropy loss function.
Specifically, we define the loss function $d$ in Equation~\eqref{eq:sample_approx} as
\begin{align}
    d(\bm{A},\hat{\bm{A}}) := \frac{1}{m}\sum^m_{j=1} \ell_{\mathrm{CE}}(\bm{A}[:,j], \hat{\bm{A}}[:,j])\label{eq:celoss},
\end{align}
where $\ell_{\mathrm{CE}}(\bm{y},\hat{\bm{y}})$ represents the cross-entropy loss between the one-hot vector $\bm{y}$ and a probability vector $\hat{\bm{y}}$.
We back-propagate this error through $\softRR$ and optimize the parameters using a standard gradient decent method.


\section{Experiments}
\begin{figure*}[t]
    \centering
    \includegraphics[width=0.92\linewidth]{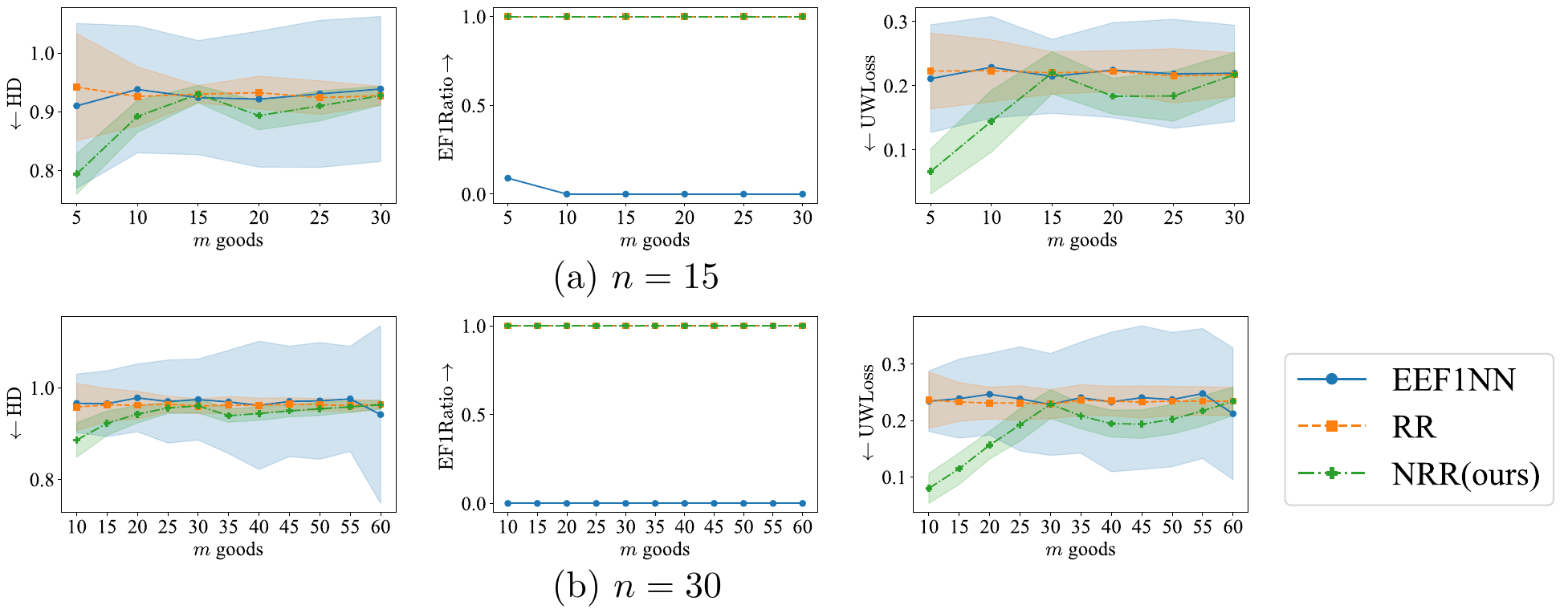}
    \caption{Evaluation metrics for varying numbers of agents and goods. 
    (a): Results for $n=15$ agents.
    (b): Results for $n=30$ agents.
    The horizontal axis represents the number of goods $m$.
    The vertical axes in each figure correspond to the following metrics: Hamming distance (leftmost), ratio of EF1 allocations (middle), and utilitarian welfare loss (rightmost).
    The symbols $\downarrow$ and $\uparrow$ indicate that the metric improves as the value decreases and increases, respectively.}
    \label{fig:results}
    \Description{Experimental results.}
\end{figure*}
\begin{figure*}[t]
    \centering
    \includegraphics[width=0.92\linewidth]{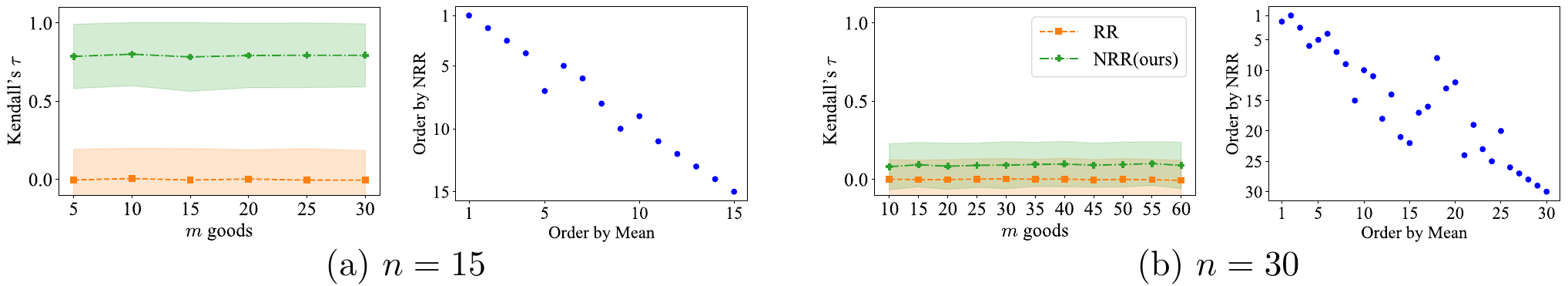}
    \caption{Kendall's $\tau$ and examples of learned agent orders.
    (a): The results for $n=15$ agents.
    (b): The results for $n=30$ agents.
    For each panel, the first figure shows Kendall's $\tau$ between learned orders and ones calculated by mean valuations.
    The legend is shared between the two panels.
    The second figure shows an example of the agent orders of mean valuation and NRR.
    The point $(x,y)$ means that the agent at rank $x$ in the mean valuation order is placed at rank $y$ in the learned order.}
    \label{fig:orders}
    \Description{Learned orders.}
\end{figure*}
\begin{figure}
    \centering
    \includegraphics[width=0.8\linewidth]{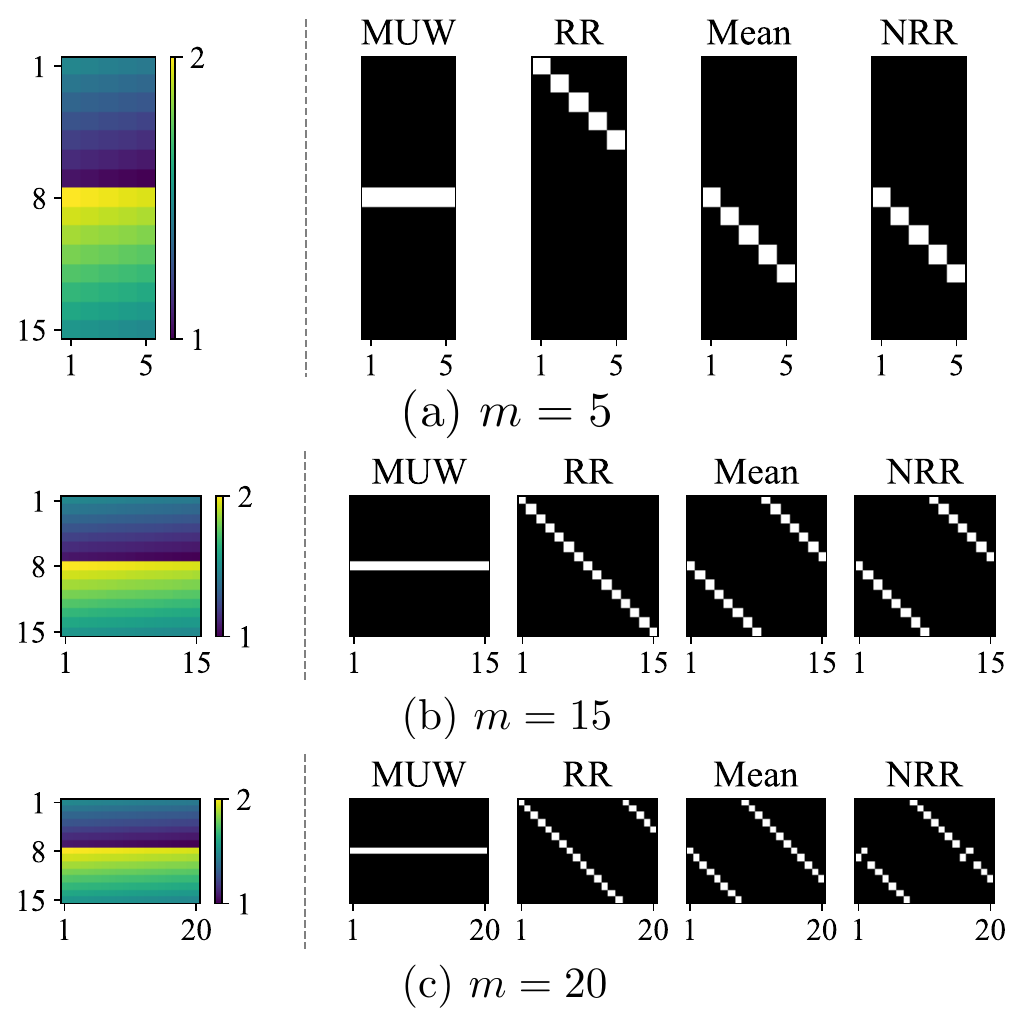}
    \caption{Examples of allocations. 
    The leftmost heat map represents the valuation profile. 
    The remaining four columns correspond to allocation results by MUW, RR, RR induced by the highest mean valuation order, and pre-trained NRR.
    In the heatmap, areas with a value of $0$ are represented in black, while areas with a value of $1$ are represented in white.}
    \label{fig:allocation}
    \Description{Examples of allocations.}
\end{figure}
We conducted experiments to compare the effectiveness of our proposed method with baseline methods in learning EF1 allocation mechanisms from examples.

\subsection{Experimental Setting}
\subsubsection{Synthetic Data}
We synthesized datasets for good allocations by modeling $i$'s valuation of good $j$ as $v_{ij} = \mu_i + \varepsilon_{i,j}$, where $\mu_i$ and $\varepsilon_{i,j}$ denote the average valuation of agent $i$ and random error, respectively.
We sampled $\mu_i$ and $\varepsilon_{i,j}$ independently and identically from $U[1,2]$ and $U[0,0.01]$, respectively.
This low-rank valuation model is commonly used to represent human preferences as seen in recommender systems~\cite{koren2009matrix}.
We generated $100$ samples  for the training, validation, and test datasets.

We synthesized allocation results using the maximum utilitarian welfare (MUW) rule as an implicit allocation rule. 
This mechanism was selected for comparison with other methods across several metrics in addition to proximity to the correct allocation results.
The MUW rule outputs the allocation that maximizes the sum of agent valuations:
\begin{align*}
    \mathrm{MUW}(\bm{V}) := \argmax\{\mathrm{UW}(\bm{V},\bm{A})\mid\text{$\bm{A}\in\{0,1\}^{n\times m}$ is an allocation}\},
\end{align*}
where $\mathrm{UW}(\bm{V},\bm{A}) := \sum^n_{i=1} v_i(A_i)$ denotes the utilitarian welfare achieved by allocation $\bm{A}$.
We computed the solutions using the Gurobi optimizer~\cite{gurobi}.
For training and validation data, we set the number of agents $n=15$ and the number of goods $m=5$, or $n=30$ and $m=10$.
We tested models trained on $n=15, m=5$ using datasets consisting of $n=15$ agents and $m = 5k~(1\le k\le 6)$.
For models trained on $n=30$ and $m=10$, we tested them using datasets with $n=30$ agents and $m = 5k~(2\le k\le 12)$.

\subsubsection{Baselines}
We use two models as baselines for comparison.
The first is the original RR, which does not reorder the agents.
The second is EEF1NN, a fully convolutional neural network proposed by \citet{mishra2022eef1nn}.
EEF1NN takes a valuation matrix as input and outputs a fractional allocation during training or an integral allocation during inference.
Because EEF1NN does not inherently satisfy EF1, \citet{mishra2022eef1nn} introduced a loss function that includes the EF violation.
In line with this, we define the loss function by adding the EF violation to the cross-entropy loss in Equation~\eqref{eq:celoss}:
\begin{align*}
    \mathrm{Loss}(\bm{A},\hat{\bm{A}}) := d(\bm{A},\hat{\bm{A}}) + \frac{\lambda}{n}\sum^{n}_{i=1}\mathrm{Envy}_i(\hat{\bm{A}}),
\end{align*}
where $\lambda$ is a Lagrangian multiplier and $\mathrm{envy}_i(\hat{\bm{A}})$ represents the sum of envy of agent $i$ toward other agents:
\begin{align*}
\mathrm{Envy}_i(\hat{\bm{A}}) := \sum_{i'=1}^n\max\{0,v_{i}(\hat{A}_{i'}) - v_i(\hat{A}_i)\}.
\end{align*}

\subsubsection{Evaluation Metrics}
We used three evaluation metrics to assess the performance of the models.

First, we computed Hamming distance (HD):
\begin{align*}
    \mathrm{HD}(\bm{A},\hat{\bm{A}}) := \frac{1}{2m}\sum_{i=1}^n\sum_{j=1}^m |A_{ij} - \hat{A}_{ij}|.
\end{align*}
Using HD, we evaluated the ability of $\neuralRR$ to predict allocations that are close to the correct ones.
We multiplied by $1/(2m)$ for normalization, because the maximum value is $2m$.

Second, following existing work~\cite{mishra2022eef1nn}, we evaluated the ratio of the number of predicted allocations that are EF1 to the total number of test samples.
We denote this ratio by
\begin{align*}
    \mathrm{EF1Ratio} := |\{\hat{\bm{A}}\mid\text{$\hat{\bm{A}}$ is EF1}\}| / (\text{\# Test instances}).
\end{align*}

Third, we calculated welfare loss.
Because we used MUW to generate the target allocations, we evaluated how much welfare the predicted allocations gained.
Formally, we defined the utilitarian welfare loss (UWLoss) as 
\begin{align*}
    \mathrm{UWLoss}(\bm{V}, \hat{\bm{A}}) := 1 - \mathrm{UW}(\bm{V},\hat{\bm{A}})/\mathrm{MUW}(\bm{V}).
\end{align*}

Note that we did not train the models incorporating the welfare loss.
This is because our motivation is to test whether we can learn implicit allocation rules through examples, which we instantiated as MUW in this experiment.
That is, we cannot explicitly formulate the implicit rules in practice, and thus cannot a priori know if such rules consider the welfare function.
However, our model is independent of loss functions, and therefore we can also incorporate such welfare metrics into the loss function in Equation~\eqref{eq:sample_approx}.

\subsubsection{Hyper-parameters}
For $\neuralRR$, we fixed the rank to $3$ for the agent low-rank embedding dimension and trained for $20$ epochs with a batch size of $4$.
We selected the two temperature parameters $\tau, \tau'$ for both $\softRR$ and SoftSort from $\{1.0, 0.1, 0.01\}$, which yielded the minimum average HD on the validation data.
We excluded hyper-parameters that did not result in a decrease in training loss.
For EEF1NN, we trained it for $20$ epochs with batch size of $4$ and set $\lambda=1.0$, which is the midpoint of the range $[0.1, 2.0]$ considered in prior research~\cite{mishra2022eef1nn}.

\subsection{Results}
Figure~\ref{fig:results} shows the evaluation metrics for various numbers of agents and goods.

For the results with $n=15$ agents, our proposed method yielded a lower HD compared to RR and EEF1NN.
Notably, the difference between $\neuralRR$ and RR was most pronounced in the ranges where $m < n$ and $n < m < 2n$, with the difference being larger in the $m < n$ case than in $n < m < 2n$.
RR and NRR are EF1 allocation mechanisms by construction, and EF1Ratio remained $1.0$.
In contrast, EEF1NN failed to output EF1 allocations in nearly all cases as reported in prior work~\cite{mishra2022eef1nn}.
UWLoss exhibited a similar pattern to HD.
The results for $n=30$ showed similar trends to those for $n=15$, with the performance gap between RR and $\neuralRR$ narrowing compared to the $n = 15$ case.


We examined the agent orders learned by $\neuralRR$, as described in Figure~\ref{fig:orders}.
For each valuation profile $\bm{V}$ in the test dataset, we compared the predicted order $\hat{\bm{P}} = \pi^{\bm{\theta}}(\bm{V})$ in $\neuralRR$ with the order based on the mean valuation $\left[m^{-1}\sum_jv_{ij}\right]_i$.
Specifically, we calculated Kendall's $\tau$ between the two orders for test instances and compared it to the fixed order $1,2,\dots,n$ used in RR.
The results confirmed that, on average, $\neuralRR$ computed orders closer to mean valuations than the fixed order used by RR.
Furthermore, we visualized an example of the two orders.
We analyzed the changes in agent orders from the mean valuation order to the learned order by plotting points $(x,y)$, where the agent ranked at position $x$ in the mean valuation order is positioned at rank $y$ in the learned order.
In these examples, the learned order closely aligned with the mean valuation, specifically at the higher and lower ranks.

\subsection{Discussion}
The experimental results demonstrate that our model outperforms the baselines across all the three evaluation metrics.

The improvements in HD and UWLoss can be attributed to both the data characteristics and $\neuralRR$'s ability to optimize the agent order. 
Because valuations were generated as $v_{ij} = \mu_i + \varepsilon_{ij}$, each training allocation typically favors the agent $i^* = \argmax_i\{\mu_i\}$ with the highest average valuation.
Thus, optimizing the agent order to place $i^*$ near the top helps approximate RR effectively for the MUW rule.
Specifically, when $kn < m < (k+1)n$ for some $k\in\N$, the optimized order can allocate $(k+1)$ goods to $i^*$, while a random order may only allocate $k$ goods.
Conversely, when $m = kn$ for some $k\in \N$, executing RR with random orders can yield the same Hamming distance because it allocates $k$ goods to $i^*$ regardless of the agent order.
Additionally, as $k$ increases, the difference between $k+1$ and $k$ diminishes, making HD between the two orders converge.
UWLoss correlates with the characteristics of HD because as more goods are allocated to $i^*$, UWLoss increases.
This relationship is illustrated in Figure~\ref{fig:allocation}.
We created a valuation profile $\bm{V}$ where $n=15$ and $i^*=7$ with $v_{i,1}>\cdots>v_{i,m}$ for all $i$.
When $m\neq kn$ for any $k$, RR induced by the highest mean valuation gives more goods to $i^*$ than RR induced by the original agent indices, while both give equal number of goods when $n=m$.
Note that we did not include RR based on the mean valuations because the experiments assume we cannot know in advance that the MUW rule is the implicit rule; therefore, we cannot predefine the order.

As shown in Figure~\ref{fig:orders}, $\neuralRR$ can learn the order based on mean valuation, resulting in similar performance to the order. 
The performance gap between RR and $\neuralRR$ narrowed for $n = 30$ compared to $n = 15$, as $\neuralRR$'s order estimation ability declined when $n = 30$. 
Further improvements for larger $n$ values are left for future work.

\section{Conclusion and Limitation}
We studied learning EF1 allocation mechanisms through examples based on implicit rules.
We first developed SoftRR for differentiable relaxation of RR, and proposed a neural network called NRR based on SoftRR.
We conducted experiments with synthetic data and compared NRR to baselines.
Experimental results show that our architecture can learn implicit rules by optimizing agent orders.
Improvement for larger number of agents are left future work.

\balance




\begin{acks}
This work is supported by JSPS KAKENHI Grant Number 21H04979 and JST, PRESTO Grant Number JPMJPR20C5.
\end{acks}



\bibliographystyle{ACM-Reference-Format} 
\bibliography{main}


\begin{thebibliography}{40}


\ifx \showCODEN    \undefined \def \showCODEN     #1{\unskip}     \fi
\ifx \showDOI      \undefined \def \showDOI       #1{#1}\fi
\ifx \showISBNx    \undefined \def \showISBNx     #1{\unskip}     \fi
\ifx \showISBNxiii \undefined \def \showISBNxiii  #1{\unskip}     \fi
\ifx \showISSN     \undefined \def \showISSN      #1{\unskip}     \fi
\ifx \showLCCN     \undefined \def \showLCCN      #1{\unskip}     \fi
\ifx \shownote     \undefined \def \shownote      #1{#1}          \fi
\ifx \showarticletitle \undefined \def \showarticletitle #1{#1}   \fi
\ifx \showURL      \undefined \def \showURL       {\relax}        \fi
\providecommand\bibfield[2]{#2}
\providecommand\bibinfo[2]{#2}
\providecommand\natexlab[1]{#1}
\providecommand\showeprint[2][]{arXiv:#2}

\bibitem[\protect\citeauthoryear{Amanatidis, Aziz, Birmpas, Filos-Ratsikas, Li, Moulin, Voudouris, and Wu}{Amanatidis et~al\mbox{.}}{2023}]%
        {amanatidis2023fair}
\bibfield{author}{\bibinfo{person}{Georgios Amanatidis}, \bibinfo{person}{Haris Aziz}, \bibinfo{person}{Georgios Birmpas}, \bibinfo{person}{Aris Filos-Ratsikas}, \bibinfo{person}{Bo Li}, \bibinfo{person}{Herv^^c3^^a9 Moulin}, \bibinfo{person}{Alexandros~A. Voudouris}, {and} \bibinfo{person}{Xiaowei Wu}.} \bibinfo{year}{2023}\natexlab{}.
\newblock \showarticletitle{Fair Division of Indivisible Goods: Recent Progress and Open Questions}.
\newblock \bibinfo{journal}{\emph{Artificial Intelligence}}  \bibinfo{volume}{322} (\bibinfo{year}{2023}), \bibinfo{pages}{103965}.
\newblock


\bibitem[\protect\citeauthoryear{Amanatidis, Markakis, Nikzad, and Saberi}{Amanatidis et~al\mbox{.}}{2017}]%
        {amanatidis2017approximation}
\bibfield{author}{\bibinfo{person}{Georgios Amanatidis}, \bibinfo{person}{Evangelos Markakis}, \bibinfo{person}{Afshin Nikzad}, {and} \bibinfo{person}{Amin Saberi}.} \bibinfo{year}{2017}\natexlab{}.
\newblock \showarticletitle{Approximation Algorithms for Computing Maximin Share Allocations}.
\newblock \bibinfo{journal}{\emph{ACM Transactions on Algorithms}}  \bibinfo{volume}{13} (\bibinfo{year}{2017}), \bibinfo{pages}{1--28}.
\newblock


\bibitem[\protect\citeauthoryear{Aziz, Caragiannis, Igarashi, and Walsh}{Aziz et~al\mbox{.}}{2019}]%
        {aziz2019fair}
\bibfield{author}{\bibinfo{person}{Haris Aziz}, \bibinfo{person}{Ioannis Caragiannis}, \bibinfo{person}{Ayumi Igarashi}, {and} \bibinfo{person}{Toby Walsh}.} \bibinfo{year}{2019}\natexlab{}.
\newblock \showarticletitle{Fair Allocation of Indivisible Goods and Chores}. In \bibinfo{booktitle}{\emph{Proceedings of the Twenty-Eighth International Joint Conference on Artificial Intelligence}}. \bibinfo{publisher}{International Joint Conferences on Artificial Intelligence Organization}, \bibinfo{pages}{53--59}.
\newblock


\bibitem[\protect\citeauthoryear{Aziz, Li, Moulin, and Wu}{Aziz et~al\mbox{.}}{2022}]%
        {aziz2022algorithmic}
\bibfield{author}{\bibinfo{person}{Haris Aziz}, \bibinfo{person}{Bo Li}, \bibinfo{person}{Herv\'{e} Moulin}, {and} \bibinfo{person}{Xiaowei Wu}.} \bibinfo{year}{2022}\natexlab{}.
\newblock \showarticletitle{Algorithmic Fair Allocation of Indivisible Items: A Survey and New Questions}.
\newblock \bibinfo{journal}{\emph{ACM SIGecom Exchanges}} \bibinfo{volume}{20}, \bibinfo{number}{1} (\bibinfo{year}{2022}), \bibinfo{pages}{24--40}.
\newblock


\bibitem[\protect\citeauthoryear{Brandt, Conitzer, Endriss, Lang, and Procaccia}{Brandt et~al\mbox{.}}{2016}]%
        {brandt2016handbook}
\bibfield{author}{\bibinfo{person}{Felix Brandt}, \bibinfo{person}{Vincent Conitzer}, \bibinfo{person}{Ulle Endriss}, \bibinfo{person}{J{\'e}r{\^o}me Lang}, {and} \bibinfo{person}{Ariel~D Procaccia}.} \bibinfo{year}{2016}\natexlab{}.
\newblock \bibinfo{booktitle}{\emph{Handbook of Computational Social Choice}}.
\newblock \bibinfo{publisher}{Cambridge University Press}.
\newblock


\bibitem[\protect\citeauthoryear{Budish}{Budish}{2011}]%
        {budish2011combinatorial}
\bibfield{author}{\bibinfo{person}{Eric Budish}.} \bibinfo{year}{2011}\natexlab{}.
\newblock \showarticletitle{The Combinatorial Assignment Problem: Approximate Competitive Equilibrium from Equal Incomes}.
\newblock \bibinfo{journal}{\emph{Journal of Political Economy}} \bibinfo{volume}{119}, \bibinfo{number}{6} (\bibinfo{year}{2011}), \bibinfo{pages}{1061--1103}.
\newblock


\bibitem[\protect\citeauthoryear{Caragiannis, Kurokawa, Moulin, Procaccia, Shah, and Wang}{Caragiannis et~al\mbox{.}}{2016}]%
        {caragiannis2016unreasonable}
\bibfield{author}{\bibinfo{person}{Ioannis Caragiannis}, \bibinfo{person}{David Kurokawa}, \bibinfo{person}{Herv\'{e} Moulin}, \bibinfo{person}{Ariel~D. Procaccia}, \bibinfo{person}{Nisarg Shah}, {and} \bibinfo{person}{Junxing Wang}.} \bibinfo{year}{2016}\natexlab{}.
\newblock \showarticletitle{The Unreasonable Fairness of Maximum Nash Welfare}. In \bibinfo{booktitle}{\emph{Proceedings of the 2016 ACM Conference on Economics and Computation}}. \bibinfo{publisher}{Association for Computing Machinery}, \bibinfo{address}{New York, NY, USA}, \bibinfo{pages}{305--322}.
\newblock


\bibitem[\protect\citeauthoryear{Conitzer, Freeman, and Shah}{Conitzer et~al\mbox{.}}{2017}]%
        {conitzer2017fair}
\bibfield{author}{\bibinfo{person}{Vincent Conitzer}, \bibinfo{person}{Rupert Freeman}, {and} \bibinfo{person}{Nisarg Shah}.} \bibinfo{year}{2017}\natexlab{}.
\newblock \showarticletitle{Fair Public Decision Making}. In \bibinfo{booktitle}{\emph{Proceedings of the 2017 ACM Conference on Economics and Computation}}. \bibinfo{publisher}{Association for Computing Machinery}, \bibinfo{address}{New York, NY, USA}, \bibinfo{pages}{629--646}.
\newblock


\bibitem[\protect\citeauthoryear{Conitzer and Sandholm}{Conitzer and Sandholm}{2002}]%
        {conitzer2002complexity}
\bibfield{author}{\bibinfo{person}{Vincent Conitzer} {and} \bibinfo{person}{Tuomas Sandholm}.} \bibinfo{year}{2002}\natexlab{}.
\newblock \showarticletitle{Complexity of Mechanism Design}. In \bibinfo{booktitle}{\emph{Proceedings of the Eighteenth Conference on Uncertainty in Artificial Intelligence}}. \bibinfo{publisher}{Morgan Kaufmann Publishers Inc.}, \bibinfo{address}{San Francisco, CA, USA}, \bibinfo{pages}{103--110}.
\newblock


\bibitem[\protect\citeauthoryear{Conitzer and Sandholm}{Conitzer and Sandholm}{2004}]%
        {conitzer2004self}
\bibfield{author}{\bibinfo{person}{Vincent Conitzer} {and} \bibinfo{person}{Tuomas Sandholm}.} \bibinfo{year}{2004}\natexlab{}.
\newblock \showarticletitle{Self-Interested Automated Mechanism Design and Implications for Optimal Combinatorial Auctions}. In \bibinfo{booktitle}{\emph{Proceedings of the 5th ACM Conference on Electronic Commerce}}. \bibinfo{publisher}{Association for Computing Machinery}, \bibinfo{address}{New York, NY, USA}, \bibinfo{pages}{132--141}.
\newblock


\bibitem[\protect\citeauthoryear{Curry, Sandholm, and Dickerson}{Curry et~al\mbox{.}}{2023}]%
        {curry2023differentiable}
\bibfield{author}{\bibinfo{person}{Michael Curry}, \bibinfo{person}{Tuomas Sandholm}, {and} \bibinfo{person}{John Dickerson}.} \bibinfo{year}{2023}\natexlab{}.
\newblock \showarticletitle{Differentiable Economics for Randomized Affine Maximizer Auctions}. In \bibinfo{booktitle}{\emph{Proceedings of the Thirty-Second International Joint Conference on Artificial Intelligence}}. \bibinfo{publisher}{International Joint Conferences on Artificial Intelligence Organization}, \bibinfo{pages}{2633--2641}.
\newblock


\bibitem[\protect\citeauthoryear{Duan, Sun, Chen, and Deng}{Duan et~al\mbox{.}}{2023}]%
        {duan2023scalable}
\bibfield{author}{\bibinfo{person}{Zhijian Duan}, \bibinfo{person}{Haoran Sun}, \bibinfo{person}{Yurong Chen}, {and} \bibinfo{person}{Xiaotie Deng}.} \bibinfo{year}{2023}\natexlab{}.
\newblock \showarticletitle{A Scalable Neural Network for DSIC Affine Maximizer Auction Design}. In \bibinfo{booktitle}{\emph{Advances in Neural Information Processing Systems}}. \bibinfo{publisher}{Curran Associates, Inc.}, \bibinfo{pages}{56169--56185}.
\newblock


\bibitem[\protect\citeauthoryear{Duan, Tang, Yin, Feng, Yan, Zaheer, and Deng}{Duan et~al\mbox{.}}{2022}]%
        {duan2022context}
\bibfield{author}{\bibinfo{person}{Zhijian Duan}, \bibinfo{person}{Jingwu Tang}, \bibinfo{person}{Yutong Yin}, \bibinfo{person}{Zhe Feng}, \bibinfo{person}{Xiang Yan}, \bibinfo{person}{Manzil Zaheer}, {and} \bibinfo{person}{Xiaotie Deng}.} \bibinfo{year}{2022}\natexlab{}.
\newblock \showarticletitle{A Context-Integrated Transformer-Based Neural Network for Auction Design}. In \bibinfo{booktitle}{\emph{Proceedings of the 39th International Conference on Machine Learning}}. \bibinfo{publisher}{PMLR}, \bibinfo{pages}{5609--5626}.
\newblock


\bibitem[\protect\citeauthoryear{Duetting, Feng, Narasimhan, Parkes, and Ravindranath}{Duetting et~al\mbox{.}}{2019}]%
        {duetting2019optimal}
\bibfield{author}{\bibinfo{person}{Paul Duetting}, \bibinfo{person}{Zhe Feng}, \bibinfo{person}{Harikrishna Narasimhan}, \bibinfo{person}{David Parkes}, {and} \bibinfo{person}{Sai~Srivatsa Ravindranath}.} \bibinfo{year}{2019}\natexlab{}.
\newblock \showarticletitle{Optimal Auctions through Deep Learning}. In \bibinfo{booktitle}{\emph{Proceedings of the 36th International Conference on Machine Learning}}. \bibinfo{publisher}{PMLR}, \bibinfo{pages}{1706--1715}.
\newblock


\bibitem[\protect\citeauthoryear{Feng, Narasimhan, and Parkes}{Feng et~al\mbox{.}}{2018}]%
        {feng2018deep}
\bibfield{author}{\bibinfo{person}{Zhe Feng}, \bibinfo{person}{Harikrishna Narasimhan}, {and} \bibinfo{person}{David~C. Parkes}.} \bibinfo{year}{2018}\natexlab{}.
\newblock \showarticletitle{Deep Learning for Revenue-Optimal Auctions with Budgets}. In \bibinfo{booktitle}{\emph{Proceedings of the 17th International Conference on Autonomous Agents and MultiAgent Systems}}. \bibinfo{publisher}{International Foundation for Autonomous Agents and Multiagent Systems}, \bibinfo{address}{Richland, SC}, \bibinfo{pages}{354--362}.
\newblock


\bibitem[\protect\citeauthoryear{Foley}{Foley}{1966}]%
        {foley1966resource}
\bibfield{author}{\bibinfo{person}{Duncan~Karl Foley}.} \bibinfo{year}{1966}\natexlab{}.
\newblock \bibinfo{booktitle}{\emph{Resource Allocation and the Public Sector}}.
\newblock \bibinfo{publisher}{Yale University}.
\newblock


\bibitem[\protect\citeauthoryear{Goldman and Procaccia}{Goldman and Procaccia}{2015}]%
        {goldman2015spliddit}
\bibfield{author}{\bibinfo{person}{Jonathan Goldman} {and} \bibinfo{person}{Ariel~D. Procaccia}.} \bibinfo{year}{2015}\natexlab{}.
\newblock \showarticletitle{Spliddit: Unleashing Fair Division Algorithms}.
\newblock \bibinfo{journal}{\emph{ACM SIGecom Exchanges}} \bibinfo{volume}{13}, \bibinfo{number}{2} (\bibinfo{year}{2015}), \bibinfo{pages}{41--46}.
\newblock


\bibitem[\protect\citeauthoryear{Golowich, Narasimhan, and Parkes}{Golowich et~al\mbox{.}}{2018}]%
        {golowich2018deep}
\bibfield{author}{\bibinfo{person}{Noah Golowich}, \bibinfo{person}{Harikrishna Narasimhan}, {and} \bibinfo{person}{David~C. Parkes}.} \bibinfo{year}{2018}\natexlab{}.
\newblock \showarticletitle{Deep Learning for Multi-Facility Location Mechanism Design}. In \bibinfo{booktitle}{\emph{Proceedings of the Twenty-Seventh International Joint Conference on Artificial Intelligence}}. \bibinfo{publisher}{International Joint Conferences on Artificial Intelligence Organization}, \bibinfo{pages}{261--267}.
\newblock


\bibitem[\protect\citeauthoryear{Gordon-Hecker, Choshen-Hillel, Shalvi, and Bereby-Meyer}{Gordon-Hecker et~al\mbox{.}}{2017}]%
        {gordon2017resource}
\bibfield{author}{\bibinfo{person}{Tom Gordon-Hecker}, \bibinfo{person}{Shoham Choshen-Hillel}, \bibinfo{person}{Shaul Shalvi}, {and} \bibinfo{person}{Yoella Bereby-Meyer}.} \bibinfo{year}{2017}\natexlab{}.
\newblock \bibinfo{booktitle}{\emph{Resource Allocation Decisions: When Do We Sacrifice Efficiency in the Name of Equity?}}
\newblock \bibinfo{publisher}{Springer International Publishing}, Chapter~6, \bibinfo{pages}{93--105}.
\newblock


\bibitem[\protect\citeauthoryear{Gudes, Kuflik, and Meisels}{Gudes et~al\mbox{.}}{1990}]%
        {gudes1990resource}
\bibfield{author}{\bibinfo{person}{Ehud Gudes}, \bibinfo{person}{Tsvi Kuflik}, {and} \bibinfo{person}{Amnon Meisels}.} \bibinfo{year}{1990}\natexlab{}.
\newblock \showarticletitle{On Resource Allocation by an Expert System}.
\newblock \bibinfo{journal}{\emph{Engineering Applications of Artificial Intelligence}} \bibinfo{volume}{3}, \bibinfo{number}{2} (\bibinfo{year}{1990}), \bibinfo{pages}{101--109}.
\newblock


\bibitem[\protect\citeauthoryear{{Gurobi Optimization, LLC}}{{Gurobi Optimization, LLC}}{2024}]%
        {gurobi}
\bibfield{author}{\bibinfo{person}{{Gurobi Optimization, LLC}}.} \bibinfo{year}{2024}\natexlab{}.
\newblock \bibinfo{title}{{Gurobi Optimizer Reference Manual}}.
\newblock
\newblock
\urldef\tempurl%
\url{https://www.gurobi.com}
\showURL{%
\tempurl}


\bibitem[\protect\citeauthoryear{Koren, Bell, and Volinsky}{Koren et~al\mbox{.}}{2009}]%
        {koren2009matrix}
\bibfield{author}{\bibinfo{person}{Yehuda Koren}, \bibinfo{person}{Robert Bell}, {and} \bibinfo{person}{Chris Volinsky}.} \bibinfo{year}{2009}\natexlab{}.
\newblock \showarticletitle{Matrix Factorization Techniques for Recommender Systems}.
\newblock \bibinfo{journal}{\emph{Computer}} \bibinfo{volume}{42}, \bibinfo{number}{8} (\bibinfo{year}{2009}), \bibinfo{pages}{30--37}.
\newblock


\bibitem[\protect\citeauthoryear{Lemieux-Charles, Meslin, Aird, Baker, and Leatt}{Lemieux-Charles et~al\mbox{.}}{1993}]%
        {lemieux-charles1993ethical}
\bibfield{author}{\bibinfo{person}{Louise Lemieux-Charles}, \bibinfo{person}{Eric~M Meslin}, \bibinfo{person}{Cheryl Aird}, \bibinfo{person}{Robin Baker}, {and} \bibinfo{person}{Peggy Leatt}.} \bibinfo{year}{1993}\natexlab{}.
\newblock \showarticletitle{Ethical Issues Faced by Clinician/Managers in Resource-Allocation Decisions}.
\newblock \bibinfo{journal}{\emph{Hospital and Health Services Administration}} \bibinfo{volume}{38}, \bibinfo{number}{2} (\bibinfo{year}{1993}), \bibinfo{pages}{267--285}.
\newblock


\bibitem[\protect\citeauthoryear{Li, Vietri, Galvani, and Chapman}{Li et~al\mbox{.}}{2010}]%
        {li2010how}
\bibfield{author}{\bibinfo{person}{Meng Li}, \bibinfo{person}{Jeffrey Vietri}, \bibinfo{person}{Alison~P. Galvani}, {and} \bibinfo{person}{Gretchen~B. Chapman}.} \bibinfo{year}{2010}\natexlab{}.
\newblock \showarticletitle{How Do People Value Life?}
\newblock \bibinfo{journal}{\emph{Psychological Science}} \bibinfo{volume}{21}, \bibinfo{number}{2} (\bibinfo{year}{2010}), \bibinfo{pages}{163--167}.
\newblock


\bibitem[\protect\citeauthoryear{Lipton, Markakis, Mossel, and Saberi}{Lipton et~al\mbox{.}}{2004}]%
        {lipton2004approximately}
\bibfield{author}{\bibinfo{person}{Richard~J. Lipton}, \bibinfo{person}{Evangelos Markakis}, \bibinfo{person}{Elchanan Mossel}, {and} \bibinfo{person}{Amin Saberi}.} \bibinfo{year}{2004}\natexlab{}.
\newblock \showarticletitle{On Approximately Fair Allocations of Indivisible Goods}. In \bibinfo{booktitle}{\emph{Proceedings of the 5th ACM Conference on Electronic Commerce}}. \bibinfo{publisher}{Association for Computing Machinery}, \bibinfo{address}{New York, NY, USA}, \bibinfo{pages}{125--131}.
\newblock


\bibitem[\protect\citeauthoryear{Mishra, Padala, and Gujar}{Mishra et~al\mbox{.}}{2022}]%
        {mishra2022eef1nn}
\bibfield{author}{\bibinfo{person}{Shaily Mishra}, \bibinfo{person}{Manisha Padala}, {and} \bibinfo{person}{Sujit Gujar}.} \bibinfo{year}{2022}\natexlab{}.
\newblock \showarticletitle{EEF1-NN: Efficient and EF1 Allocations Through Neural Networks}. In \bibinfo{booktitle}{\emph{Pacific Rim International Conference on Artificial Intelligence 2022: Trends in Artificial Intelligence}}. \bibinfo{publisher}{Springer Nature Switzerland}, \bibinfo{address}{Cham}, \bibinfo{pages}{388--401}.
\newblock


\bibitem[\protect\citeauthoryear{Narasimhan, Agarwal, and Parkes}{Narasimhan et~al\mbox{.}}{2016}]%
        {narasimhan2016automated}
\bibfield{author}{\bibinfo{person}{Harikrishna Narasimhan}, \bibinfo{person}{Shivani Agarwal}, {and} \bibinfo{person}{David~C. Parkes}.} \bibinfo{year}{2016}\natexlab{}.
\newblock \showarticletitle{Automated Mechanism Design without Money via Machine Learning}. In \bibinfo{booktitle}{\emph{Proceedings of the Twenty-Fifth International Joint Conference on Artificial Intelligence}}. \bibinfo{pages}{433--439}.
\newblock


\bibitem[\protect\citeauthoryear{Othman, Sandholm, and Budish}{Othman et~al\mbox{.}}{2010}]%
        {othman2010finding}
\bibfield{author}{\bibinfo{person}{Abraham Othman}, \bibinfo{person}{Tuomas Sandholm}, {and} \bibinfo{person}{Eric Budish}.} \bibinfo{year}{2010}\natexlab{}.
\newblock \showarticletitle{Finding Approximate Competitive Equilibria: Efficient and Fair Course Allocation}. In \bibinfo{booktitle}{\emph{Proceedings of the 9th International Conference on Autonomous Agents and Multiagent Systems}}. \bibinfo{publisher}{International Foundation for Autonomous Agents and Multiagent Systems}, \bibinfo{address}{Richland, SC}, \bibinfo{pages}{873--880}.
\newblock


\bibitem[\protect\citeauthoryear{Peri, Curry, Dooley, and Dickerson}{Peri et~al\mbox{.}}{2021}]%
        {peri2021preferencenet}
\bibfield{author}{\bibinfo{person}{Neehar Peri}, \bibinfo{person}{Michael Curry}, \bibinfo{person}{Samuel Dooley}, {and} \bibinfo{person}{John Dickerson}.} \bibinfo{year}{2021}\natexlab{}.
\newblock \showarticletitle{PreferenceNet: Encoding Human Preferences in Auction Design with Deep Learning}. In \bibinfo{booktitle}{\emph{Advances in Neural Information Processing Systems}}. \bibinfo{publisher}{Curran Associates, Inc.}, \bibinfo{pages}{17532--17542}.
\newblock


\bibitem[\protect\citeauthoryear{Prillo and Eisenschlos}{Prillo and Eisenschlos}{2020}]%
        {prillo2020softsort}
\bibfield{author}{\bibinfo{person}{Sebastian Prillo} {and} \bibinfo{person}{Julian Eisenschlos}.} \bibinfo{year}{2020}\natexlab{}.
\newblock \showarticletitle{{S}oft{S}ort: A Continuous Relaxation for the argsort Operator}. In \bibinfo{booktitle}{\emph{Proceedings of the 37th International Conference on Machine Learning}}. \bibinfo{publisher}{PMLR}, \bibinfo{pages}{7793--7802}.
\newblock


\bibitem[\protect\citeauthoryear{Rahme, Jelassi, Bruna, and Weinberg}{Rahme et~al\mbox{.}}{2021}]%
        {rahme2021permutation}
\bibfield{author}{\bibinfo{person}{Jad Rahme}, \bibinfo{person}{Samy Jelassi}, \bibinfo{person}{Joan Bruna}, {and} \bibinfo{person}{S.~Matthew Weinberg}.} \bibinfo{year}{2021}\natexlab{}.
\newblock \showarticletitle{A Permutation-Equivariant Neural Network Architecture For Auction Design}. In \bibinfo{booktitle}{\emph{Proceedings of the AAAI Conference on Artificial Intelligence}}, Vol.~\bibinfo{volume}{35}. \bibinfo{publisher}{AAAI Press}, \bibinfo{pages}{5664--5672}.
\newblock


\bibitem[\protect\citeauthoryear{Ravindranath, Jiang, and Parkes}{Ravindranath et~al\mbox{.}}{2023}]%
        {ravindranath2023datamarket}
\bibfield{author}{\bibinfo{person}{Sai~Srivatsa Ravindranath}, \bibinfo{person}{Yanchen Jiang}, {and} \bibinfo{person}{David~C Parkes}.} \bibinfo{year}{2023}\natexlab{}.
\newblock \showarticletitle{Data Market Design through Deep Learning}. In \bibinfo{booktitle}{\emph{Advances in Neural Information Processing Systems}}, Vol.~\bibinfo{volume}{36}. \bibinfo{publisher}{Curran Associates, Inc.}, \bibinfo{pages}{6662--6689}.
\newblock


\bibitem[\protect\citeauthoryear{Sandholm}{Sandholm}{2003}]%
        {sandholm2003automated}
\bibfield{author}{\bibinfo{person}{Tuomas Sandholm}.} \bibinfo{year}{2003}\natexlab{}.
\newblock \showarticletitle{Automated Mechanism Design: A New Application Area for Search Algorithms}. In \bibinfo{booktitle}{\emph{Principles and Practice of Constraint Programming -- CP 2003}}. \bibinfo{publisher}{Springer Berlin Heidelberg}, \bibinfo{address}{Berlin, Heidelberg}, \bibinfo{pages}{19--36}.
\newblock


\bibitem[\protect\citeauthoryear{Sandholm and Likhodedov}{Sandholm and Likhodedov}{2015}]%
        {sandholm2015automated}
\bibfield{author}{\bibinfo{person}{Tuomas Sandholm} {and} \bibinfo{person}{Anton Likhodedov}.} \bibinfo{year}{2015}\natexlab{}.
\newblock \showarticletitle{Automated Design of Revenue-Maximizing Combinatorial Auctions}.
\newblock \bibinfo{journal}{\emph{Operations Research}} \bibinfo{volume}{63}, \bibinfo{number}{5} (\bibinfo{year}{2015}), \bibinfo{pages}{1000--1025}.
\newblock


\bibitem[\protect\citeauthoryear{Shaikh}{Shaikh}{2020}]%
        {shaikh2020artificial}
\bibfield{author}{\bibinfo{person}{Sonia~Jawaid Shaikh}.} \bibinfo{year}{2020}\natexlab{}.
\newblock \showarticletitle{Artificial Intelligence and Resource Allocation in Healthcare: The Process-Outcome Divide in Perspectives on Moral Decision-Making}. In \bibinfo{booktitle}{\emph{Proceedings of the {AAAI} Fall Symposium on {AI} for Social Good}}.
\newblock


\bibitem[\protect\citeauthoryear{Shen, Tang, and Zuo}{Shen et~al\mbox{.}}{2019}]%
        {shen2019automated}
\bibfield{author}{\bibinfo{person}{Weiran Shen}, \bibinfo{person}{Pingzhong Tang}, {and} \bibinfo{person}{Song Zuo}.} \bibinfo{year}{2019}\natexlab{}.
\newblock \showarticletitle{Automated Mechanism Design via Neural Networks}. In \bibinfo{booktitle}{\emph{Proceedings of the 18th International Conference on Autonomous Agents and MultiAgent Systems}}. \bibinfo{publisher}{International Foundation for Autonomous Agents and Multiagent Systems}, \bibinfo{address}{Richland, SC}, \bibinfo{pages}{215--223}.
\newblock


\bibitem[\protect\citeauthoryear{Wagstaff, Fuchs, Engelcke, Osborne, and Posner}{Wagstaff et~al\mbox{.}}{2022}]%
        {wagstaff2022universal}
\bibfield{author}{\bibinfo{person}{Edward Wagstaff}, \bibinfo{person}{Fabian~B. Fuchs}, \bibinfo{person}{Martin Engelcke}, \bibinfo{person}{Michael~A. Osborne}, {and} \bibinfo{person}{Ingmar Posner}.} \bibinfo{year}{2022}\natexlab{}.
\newblock \showarticletitle{Universal Approximation of Functions on Sets}.
\newblock \bibinfo{journal}{\emph{Journal of Machine Learning Research}} \bibinfo{volume}{23}, \bibinfo{number}{151} (\bibinfo{year}{2022}), \bibinfo{pages}{1--56}.
\newblock


\bibitem[\protect\citeauthoryear{Wang, Duetting, Ivanov, Talgam-Cohen, and Parkes}{Wang et~al\mbox{.}}{2023}]%
        {wang2023contract}
\bibfield{author}{\bibinfo{person}{Tonghan Wang}, \bibinfo{person}{Paul Duetting}, \bibinfo{person}{Dmitry Ivanov}, \bibinfo{person}{Inbal Talgam-Cohen}, {and} \bibinfo{person}{David~C Parkes}.} \bibinfo{year}{2023}\natexlab{}.
\newblock \showarticletitle{Deep Contract Design via Discontinuous Networks}. In \bibinfo{booktitle}{\emph{Advances in Neural Information Processing Systems}}. \bibinfo{publisher}{Curran Associates, Inc.}, \bibinfo{pages}{65818--65836}.
\newblock


\bibitem[\protect\citeauthoryear{Wang, Jiang, and Parkes}{Wang et~al\mbox{.}}{2024}]%
        {wang2024gemnet}
\bibfield{author}{\bibinfo{person}{Tonghan Wang}, \bibinfo{person}{Yanchen Jiang}, {and} \bibinfo{person}{David~C. Parkes}.} \bibinfo{year}{2024}\natexlab{}.
\newblock \bibinfo{title}{GemNet: Menu-Based, Strategy-Proof Multi-Bidder Auctions Through Deep Learning}.
\newblock
\newblock
\showeprint[arxiv]{2406.07428}~[cs.GT]
\urldef\tempurl%
\url{https://arxiv.org/abs/2406.07428}
\showURL{%
\tempurl}


\bibitem[\protect\citeauthoryear{Zaheer, Kottur, Ravanbakhsh, Poczos, Salakhutdinov, and Smola}{Zaheer et~al\mbox{.}}{2017}]%
        {zaheer2017deepsets}
\bibfield{author}{\bibinfo{person}{Manzil Zaheer}, \bibinfo{person}{Satwik Kottur}, \bibinfo{person}{Siamak Ravanbakhsh}, \bibinfo{person}{Barnabas Poczos}, \bibinfo{person}{Russ~R Salakhutdinov}, {and} \bibinfo{person}{Alexander~J Smola}.} \bibinfo{year}{2017}\natexlab{}.
\newblock \showarticletitle{Deep Sets}. In \bibinfo{booktitle}{\emph{Advances in Neural Information Processing Systems}}, Vol.~\bibinfo{volume}{30}. \bibinfo{publisher}{Curran Associates, Inc.}, \bibinfo{pages}{3391--3401}.
\newblock


\end{thebibliography}


\end{document}